\documentclass{svproc}
	\usepackage{url}
	
	\usepackage{etoolbox}
	\makeatletter
	\let\llncs@addcontentsline\addcontentsline
	\patchcmd{\maketitle}{\addcontentsline}{\llncs@addcontentsline}{}{}
	\makeatother
	\usepackage{hyperref}
	\usepackage{bookmark}	
	\usepackage{graphicx}
		\RequirePackage{myMac-llncs}
		\RequirePackage{lineno}	
		\RequirePackage{stackrel}

\renewcommand{\myMarginPar}[2][]{}

\newcommand{\ifarxiv}[2]{#2} 

\newcommand{\savespace}[2][]{#1}

	\renewcommand{\beqarrys}{$$\begin{array}{llll}}
	\renewcommand{\eeqarrys}{\end{array}$$}

\newcommand{\tor}{\textrm{tort}}
\newcommand{\ltor}{m}
\newcommand{\utor}{M}
\newcommand{\Wspace}{\calW}
\newcommand{\Xspace}{\calX}

\newcommand{\Ws}{\Wspace}
\newcommand{\Xs}{\Xspace}

\newcommand{\disW}{\ensuremath{d_{\Ws}}}
\newcommand{\disX}{\ensuremath{d_{\Xs}}}

\newcommand{\la}{\langle}
\newcommand{\ra}{\rangle}

\newcommand{\Riman}{\textrm{Riemannian}}

\newcommand{\Frob}{\textrm{Frobenius}}
\newcommand{\Lip}{\textrm{Lipschitz}}
	
\newcommand{\ddt}{\textrm{dt}}

\newcommand{\dx}{\textrm{dx}}
\newcommand{\dy}{\textrm{dy}}
\newcommand{\dz}{\textrm{dz}}

\newcommand{\txt}{\textrm}
\newcommand{\refThem}{\txt{Theorem }\ref}
\newcommand{\refEq}{\txt{Equation }\ref}

\newcommand{\cs}{\cspace}

\newcommand{\lto}[1][]{\stackrel{#1}{\leadsto}}

\renewcommand{\ifarxiv}[2]{#1}

	\newcommand{\cspace}{\ensuremath{\mathcal{C}{space}}}

	\newcommand{\sqs}[1][2]{{\wh{S^{#1}}}} 
	\newcommand{\whso}[1][3]{{\wh{SO}_{#1}}} 
	
	\newcommand{\WH}{^{\wh{\;}}}

	\newcommand{\ccc}{C}

\begin{document}
\mainmatter              
\title{Distortion Bounds of Subdivision Models for $SO(3)$}
%

%
\titlerunning{Subdivision Atlas and Distortion for $SO(3)$}
%

%
\author{Zhaoqi Zhang\inst{1}
	\and Chee Yap\inst{2}}
\authorrunning{Zhang \& Yap} 

\author{Chee Yap\thanks{%
	This work is supported by NSF Grant No.~CCF-2008768.
	Zhaoqi's PhD work was supported under this grant.}
\and
Zhaoqi Zhang
}
\authorrunning{Yap and Zhang}
%
%

\institute{Courant Institute, New York University\\
New York, NY 10012, USA\\
\email{yap@cs.nyu.edu \\ Homepage: 
	\texttt{https://cs.nyu.edu/yap}}
\\[2mm] \and
Beijing Key Laboratory of Topological Statistics\\
	\& Applications for Complex Systems\\
Beijing Institute of Mathematical Sciences and Applications\\
Beijing 101408, China\\
\email{zhangzhaoqi@bimsa.cn \\ Homepage:
	\texttt{https://www.bimsa.cn/detail/ZhaoqiZhang.html} }
}

\maketitle              
%
\begin{abstract}
In the subdivision approach to robot path planning,
we need to subdivide the configuration space of a robot
into nice cells to perform various computations.
For a rigid spatial robot, this configuration space
is $SE(3)=\mathbb{R}^3\times SO(3)$.  
The subdivision of $\mathbb{R}^3$ is standard but so far, there are no
global subdivision schemes for $SO(3)$.  We recently
introduced a representation for $SO(3)$
suitable for subdivision. This paper investigates
the distortion of the natural metric on $SO(3)$ caused
by our representation.  The proper framework for
this study lies in the Riemannian geometry
of $SO(3)$, enabling us to obtain exact distortion
bounds.

	\end{abstract}
	%
\sect{Introduction}
	Path planning is a fundamental task in robotics
\ifarxiv{%
	\cite{lavalle:planning:bk,choset-etal:bk}.
	}{%
	\cite{lavalle:planning:bk}.
	}
	The problem may be formulated thus:
	Fix a robot $R_0$ in $\RR^k$ ($k=2,3$).
	For example, a rigid robot $R_0$ can be identified
	as a subset of $\RR^k$ (typically a disc or convex polygon).
	Given $(\alpha,\beta,\Omega)$ where
	$\alpha,\beta$ are the start and goal configurations of $R_0$,
	the task is to
	either find a path from $\alpha$ to $\beta$ avoiding the
	obstacle set $\Omega\ib\RR^k$, or output NO-PATH.
	Such an algorithm is called a \dt{planner} for $R_0$.
	This problem originated in AI as the FINDPATH problem.
	In the 1980s, path planning began to be studied algorithmically,
\ifarxiv{%
	from an intrinsic geometric perspective
	\cite{brooks-perez:subdivision:83}.
	}{%
	from an intrinsic geometric perspective.
	}
\ifarxiv{%
	Schwartz and Sharir \cite{ss2} showed that algebraic
	}{%
	Schwartz and Sharir showed that algebraic
	}
	path planning can be solved exactly by a reduction
	to cylindrical algebraic decomposition.
\ifarxiv{%
	Yap \cite{yap:amp:87} described 
	}{%
	Yap (1987) described
	}
	two ``universal methods'' for constructing such paths:
\ifarxiv{%
	cell-decomposition \cite{ss1,ss2}
	and retraction 
		\cite{odun-yap:disc:85,odun-sharir-yap:retraction:83}.
	}{%
	cell-decomposition and retraction.
	}
	These exact methods are largely of theoretical interest
	because they require exact computation with algebraic numbers.
	Since no physical robot is exact and maps of the world
	even less so, we need numerical approximations.  But
	we lack a systematic way to
	``approximately implement exact algorithms'' (this difficulty
	is not specific to path planning).
	After the 1990s the exact approach is largely eclipsed by the
	sampling approach
\ifarxiv{%
	(combined with randomness):
	thus we have PRM \cite{prm}, RRT \cite{rrt}, etc.
	}{%
	(combined with randomness) such as
	PRM, RRT and many variants etc.
	}
	These proved to be practical, widely
\ifarxiv{%
	applicable and easy to implement \cite{choset-etal:bk}.
	}{%
	applicable and easy to implement \cite{lavalle:planning:bk}.
	}
	But it has a well-known bane called the ``narrow passage'' problem
\ifarxiv{%
	\cite{hsu-latombe-kurniawati:foundations:06}.
}{}
	This bottleneck (sic) is actually symptomatic
	of a deeper problem which we pointed out in
	\cite{sss1}: the sampling algorithms do not know how to halting
	when there is no path.  Recently,
	Dayan et al.~\cite{dayan-solovey-pavone-halperin:near-opt:21}
	provided a halting sampling algorithm for an Euclidean configuration
	space (for a collection of discs).  This requires an extra
	resolution parameter analogous to our $\veps>0$.
	Similar results could in principle, be achieved
	for sampling\footnote{
		Note that our subdivision framework could easily
		be adapted to give such halting sampling algorithms.
	} non-Euclidean configuration spaces.
	For a full discussion of this issue, see
	\cite[Appendix A]{zhang-chiang-yap:se3-arxiv:24}.

\ifarxiv{%
	In \cite{sss1,sss,sss2}, we revisited the
	}{%
	In \cite{sss1,sss2}, we revisited the
	}
	subdivision approach by introducing
	the ``Soft Subdivision Search'' (SSS) framework 
	to address two foundational issues:
	(1) To avoid the underlying cause of the above
	halting problem, we introduce the notion
	of \dt{resolution-exactness}. 
	(2) To exploit resolution-exactness, we need
	the notion of \dt{soft-predicates}.
	\ignore{
	\bitem
	\item {\em How can we avoid the ``Zero Problem''
	in standard formulations of path planning?}
	Ultimately, any planner must decide if a certain quantity is zero 
	because standard formulations contain a 
	``hard'' decision problem: is there a path or no? 
	We avoid this zero problem by introducing
	the criteria of \dt{resolution-exactness} into path planning.
	\item {\em What is a principled way to use approximations
	in path planning?}
	The SSS framework allows such approximations when we
	combine it with the use of \dt{soft predicates}.
	The latter involves numerical approximations, but
	are able to achieve certified results.
	\eitem
	}
	In a series of papers with implementations
\ifarxiv{%
		\cite{sss1,%
		yap-luo-hsu:thicklink:16,luo-chiang-lien-yap:link:14,%
		zhou-chiang-yap:complex-robot:20,%
		rod-ring,%
		zhang-chiang-yap:se3-robot:24}, 
}{
	\cite{sss1,zhang-chiang-yap:se3-robot:24}, 
}
	we showed that SSS framework is practical.
	The guarantee of resolution-exactness
	is much stronger than any guarantees of sampling algorithms.
	Despite such strong guarantees,
	SSS planners outperform or match
	the state-of-art sampling algorithms
	for various robots with up to 6 degrees of freedom (DOF).
	Our last paper 
		\cite{zhang-chiang-yap:se3-robot:24} 
	reached a well-known milestone, achieving the first
	rigorous, complete and implementable path planner
	for a rigid spatial robot with 6 DOF.


	Central to the design and implementation of the 6-DOF planner
	\cite{zhang-chiang-yap:se3-robot:24} 
	is a representation of $SO(3)$ that supports subdivision.
	As a 3-dimensional space, $SO(3)$ can be 
	{\em locally} represented by
	three real parameters.  E.g., using Euler angles
	$(\alpha,\beta,\gamma)$ which range over the box
	$B_0 = [-\pi,\pi]\times[-\pi/2,\pi/2]\times[-\pi,\pi]\ib\RR^3$.
	\savespace{E.g., see \cite{lee-choset:sensor-rod-planning:05}.}
	Such parametrizations have well-known singularities
	($\beta=0$)
	and a wrong global topology.
	For example for there is 
	    	\begin{figure}[htb]
	    	  \begin{center}
		   \scalebox{0.18}{
	    	     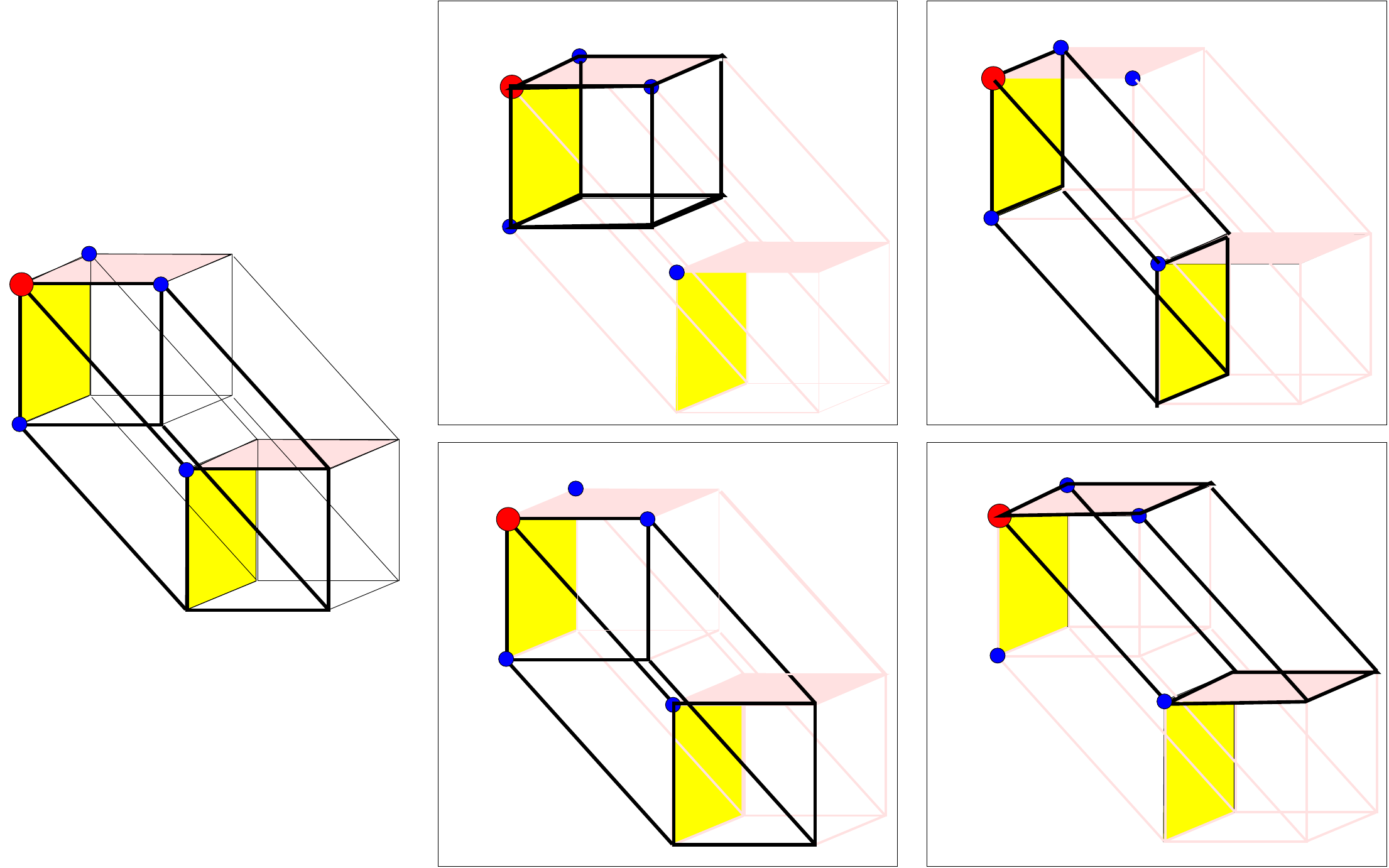}
	    	   \caption{The Cubic model $\whso$ for $SO(3)$ (taken
			   		from \cite{sss2})}
	    	   \label{fig:so3-model}
	    	  \end{center}
	    	\end{figure} 
	By viewing $SO(3)$ as quaternions,
	we can represent it by the boundary $\partial [-1,1]^4$
	of the $4$-dimensional cube $[-1,1]^4$, after
	identification of opposite pairs of faces.
	After the identification,
	\ignore{
	First represent $S^3$ as the boundary of the standard
	4-cube $[-1,1]^4$ via the map 
		\beql{whq}
			\mmatx[rl]{
			\mu:& \partial([-1,1]^4) \to S^3\\
			    & q \mapsto \frac{q}{\|q\|_2}.
			}
			\eeql
	}
	we have $4$ copies of the standard 3-cube $C=[-1,1]^3\ib \RR^3$
	which are shown as $C_w, C_x, C_y, C_z$ in
	\refFig{so3-model}.  We can now do subdivision on these cubes.
	This ``cubic model''\footnote{
		We are indebted to the late Stephen Cameron who first
		brought it to our attention (June 2018).
		This paper is dedicated to his memory.
	} of $SO(3)$ was known to Canny \cite[p.~36]{canny:thesis},
	and to Nowakiewicz \cite{nowakiewicz:mst:10}.
	To our knowledge, this model has never been systematically
	developed before.  New data structures and algorithms
	for this representation are needed
\ifarxiv{%
	\cite{zhang-chiang-yap:se3-robot:24,zhang:thesis}.
	}{%
	\cite{zhang-chiang-yap:se3-robot:24}.
	}
	This paper addresses a mathematical question about
	this representation.

\dt{Brief Literature Overview.}
	Besides the above introduction to path planning,
\ifarxiv{%
	there are many surveys such as
	\cite{hss:motionplanning:crc:17} and even meta-reviews
	\cite[Table 1]{mobileRobot-survey:21}.
	}{%
	there are many surveys
	\cite{zhang-yap:distortion-arxiv:25}.
	}
\ignore{
	Here is a brief review our work on SSS planers:
		Our first paper \cite{sss1} %
		introduce the fundamental concepts of SSS such
		as resolution-exactness and soft predicates;
		SSS planners for the disc and triangle robots were implemented.
		The papers
		\cite{yap-luo-hsu:thicklink:16,luo-chiang-lien-yap:link:14}
		address 2-link planar robot with 4DOF: novelties including
		planning in the configuration space (a torus) with a 
		diagonal band is removed,
		and extension to ``thick robots''.
		In \cite{zhou-chiang-yap:complex-robot:20}, %
		we address ``complex'' planar rigid robots, i.e.,
		robots modeled by an arbitrary simple polygon.
		The algorithms of computational geometry 
		need high combinatorial complexity, while SSS
		supports a simple decomposition technique that is linear
		in the complexity of the polygon.
		We began addressing 5DOF spatial robots in \cite{rod-ring}:
		the configuration space is $\RR^3\times S^2$ corresponds
		to axially-symmetric robots.  We designed and implemented
		planners for the rod robot (a well-studied case in theory)
		and the ring robot (first in the literature).
		Several new techniques were needed to handle
		the complexities of 3D.
}%
	In this paper, we study $SO(3)$ as a metric space.
	Since $3D$ rotations arise in applications such
	as computer vision and graphics, many $SO(3)$ metrics are known.
	Huynh \cite{huynh:metrics-rotation:09} listed
	six of these metrics $\Phi_i$ ($i=1\dd 6$).
	We will focus on $\Phi_6$, simply calling
	it the \dt{natural distance} for $SO(3)$ because
	it has all the desirable properties and respects the
	Lie group structure of $SO(3)$. 
	Basically, $\Phi_6(R_1,R_2)$ is the angle of the rotation
	$R_1\inv R_2$ about its rotation-axis.
	\savespace{
		Given rotation matrices
	$R_1, R_2$, $\Phi_6(R_1,R_2)\as \|\log(R_1R_2^T)\|$;
	equivalently, $\Phi_6(q_1,q_2)=2\arccos(|q_1\cdot q_2|)$ if
	we replace $R_i$ by the corresponding unit quaternion $q_i$
	with dot-product $q_1\cdot q_2$.
	}%
\ifarxiv{%
	Another literature on $SO(3)$ comes from
	mathematics and mechanical engineering
	(Selig \cite{selig:robotics:bk}
	and Roth and Bottema \cite{bottema-roth:bk}).
	}{%
	A slightly abbreviated version of
	this arXiv paper will appear in the 3rd 
	}

\ifarxiv{%
	A slightly abbreviated version of this paper will
	appear in the LNCS proceedings of the
	\myHref{https://ima.org.uk/25349/3rd-ima-conference-on-mathematics-of-robotic}
	{3rd IMA Conference on Mathematics of Robotics},
		Manchester, UK.  24-26 Sep, 2025.
	}{%
	{\em More complete references and any
	missing proofs in this paper may be found in the arXiv version
	of this paper}
	\cite{zhang-yap:distortion-arxiv:25}.
	}
	
	\savespace{
	Yershova et al.~\cite{yershova+3:so3-hopf:10}) addressed the
	problem of sampling $SO(3)$; this is critical to
	sampling-based planners, just as subdivision is for us.
	}
	\ignore{
	\bitem
		\item $\Phi_1$ is based on Euler angles: given two rotations
			$\bfalpha=(\phi,\theta,\psi)$,
			and $\bfalpha'=(\phi',\theta',\psi')$, we have
			$\Phi_1(\bfalpha,\bfalpha')\as\|\bfalpha-\bfalpha'\|_2$
			where $\norm{\cdot}_2$ is the Euclidean norm.
			Unfortunatedly, this parametrization 
			of $SO(3)$ has well-known singularities ($\beta=0$)
			and so we will not discuss it further.
		\item $\Phi_2(\bfq,\bfq')
				\as \min\set{\|\bfq-\bfq'\|_2,\|\bfq+\bfq'\|_2}$
			where $\bfq,\bfq'\in\RR^4$ are unit quaternions.
			This metric was used by Ravani and Roth
			\cite{ravani-roth:motion:83}
			(see also \cite{yershova+3:so3-hopf:10}).
		\item $\Phi_3(\bfq,\bfq')\as \arccos(|\bfq\cdot \bfq'|)$
		where $\bfq\cdot\bfq'$ denotes scalar product.
			It was used by Wunsch et al \cite{wunsch+2:post:97}.
		\item $\Phi_4(\bfq,\bfq')\as 1-|\bfq\cdot\bfq'|$
			was used by Kuffner \cite{kuffner:distance:04}.
		\item $\Phi_5(R_1,R_2)\as \|I -R_1R_2^T\|_F$
			where $R_1,R_2\in SO(3)$ and
			$\norm{\cdot}_F$ is the Frobenius norm on matrices.
			This was Larochelle et al.~\cite{larochelle+2:distance:07}.
		\item $\Phi_6(R_1,R_2)\as \|\log(R_1R_2^T\|$
			where the logarithm of a rotation matrix gives
			its skew-symmetric
			\cite{park:distance:95,park-ravani:smooth:97}
	\eitem
	She investigated whether these metrics are
	invariant (left-, right- or bi-invariant) 
	under the group actions in $SO(3)$,
	and whether two $\Phi_i,\Phi_j$
	are boundedly (resp., functionally) equivalent to each other.
	}
	\ignore{%
		Since $SO(3)$ is a group, it is natural to ask if metrics
		are invariant under group actions. 
	A metric $\Phi(R_1,R_2)$ is
	said to be \dt{bi-invariant} if
	$\Phi(RR_1,RR_2)=\Phi(R_1,R_2)=\Phi(R_1R,R_2R)$ 
	for all $R,R_1,R_2\in SO(3)$.  If only the left (resp., right)
	equation hold, then $\Phi$ is left (resp., right) invariant.
	Two metrics $\Phi_i,\Phi_j$ are \dt{boundedly equivalent}
	if there exists some $0<a<b$ such that
	$\Phi_i(R_1,R_2)/\Phi_j(R_1,R_2) \in [a,b]$.  Finally, they
	are \dt{functionally equivalent}
	}%


\section{Distortion in Cubic Models for $S^n$}
	We call $\sqs[n]\as \partial([-1,1]^{n+1})$ 
	the \dt{cubic model} of $S^n$ and 
	consider the homeomorphism
			$\mu_n : \sqs[n] \to S^n$
	where
		\beql{wh}
			\mu_n(q)\as q/\|q\|_{2}
			\quad\text{and}\quad
			\mu_n\inv(q) \as q/\|q\|_{\infty}
		\eeql
	and $\|q\|_p$ denotes the $p$-norm.
	Viewing $S^n$ and $\sqs[n]$ as metric spaces
	with the induced metric \cite[p.~3]{petersen:riemannian:bk},
	the \dt{cubic representation} $\mu_n$ 
	introduces a distortion in the distance function of $S^n$.
	We want to bound this distortion.

	In general, if $(X,d_X)$ and $(Y,d_Y)$ are metric spaces
	and $f:X\to Y$ is continuous, we define the
	\dt{distortion range} of $f$ to be the closure of the set
		\beql{Df}
			D_f \as \set{ \frac{d_Y(f(p),f(q))}{d_X(p,q)}:
				p,q\in X, p\neq q}.
		\eeql
	If the distortion range is $[a,b]$, then the
	\dt{distortion constant} of $f$ is
		\beql{C0}
			C_0(f)\as \max\set{b,1/a}.
		\eeql
	Note that $C_0(f)\ge 1$ is the largest
	expansion or contraction factor produced by the map $f$.
	In \cite{sss2}, $C_0(f)$ was introduced as the subdivision
	atlas constant.
	If $C_0(f)=1$, then $f$ is just an isometry.
	In \cite{rod-ring}, we showed\footnote{
		We originally claimed that the range is $[1/\sqrt{3},1]$;
		Zhaoqi pointed out that the correct range is $[1/3,1]$.
		}
	that the map $\mu_2$ has distortion range $[\tfrac{1}{3},1]$.
	The proof for $\mu_2$ used elementary
	geometry which is not easily generalized to $\mu_3$.
	In this paper, we provide the proper mathematical
	framework for a generalization to 
	any $\mu_n$.  This paper will focus on $\mu_3$
\ifarxiv{%
	(see \cite{zhang:thesis} for the general case).
	}{%
	(see Zhang's thesis for the general case).
	}
	Our main theorem is the following:

	\bthmT[mu3]{Distance Distortion Range for $\mu_3$}
			$$D_{\mu_3} = [\tfrac{1}{4},1].$$
		Thus the distortion constant for $\mu_3$ is $4$.
	\ethmT

	The key to the proof lies in exploiting the
	Riemannian metric of $S^3$ and $\sqs[3]$.
	In practical applications,
	we want representations whose distortion constant is as small
	as possible,
	subject to other considerations.
	E.g., our SSS planner \cite{zhang-chiang-yap:se3-robot:24}
	(like many algorithms)
	uses an $\veps>0$ parameter to discard
	a subdivision box $B$ if ``$\veps> width(B)$''.   
	Clearly $width(B)$ is a distorted substitute
	for distance in $SO(3)$, but how distorted is it?

	The following simple lemma is very useful:

	\blemT[composition]{Composition of Distortion Range}
		If $f:X\to Y$ and $g:Y\to Z$
		have distortion ranges of $[a,b]$
		and $[c,d]$ (respectively),
		then the distortion range of $h=g\circ f: X\to Z$ is contained
		in $[ca,db]$.  If $c=d$, then $D_h=[ca,cb]$. 
	\elemT
\ifarxiv{
	\bpf
		For $p\neq q\in X$,
			\beqarrys
			\frac{d_Z(h(p),h(q))}{d_X(p,q)}
				&=& \frac{d_Z(g(f(p)),g(f(q)))}{d_X(p,q)}\\
				&=& \frac{d_Z(g(f(p)),g(f(q)))}{d_Y(f(p),f(q))}
						\cdot
					\frac{d_Y(f(p),f(q))}{d_X(p,q)}\\
				&\in& [c,d]\cdot [a,b] = [ca, db].
			\eeqarrys
	\epf
}{
	See proof in \cite{zhang-yap:distortion-arxiv:25}.
}%

	To apply this lemma, let $f$ be the map $\mu_3$
	and for $K>0$, let
		$g_K: \partial [-K,K]^{3}\to \partial [-1,1]^{3}$
	where $g_K(q)=q/K$.  So $g_K$ has distortion range
			$[\tfrac{1}{K},\tfrac{1}{K}]$.
	Then \refLem{composition} implies that the map
			$g_K(\wh{\cdot}):\partial[-K,K]^3\to S^3$
	has distortion range $[1/4K,1/K]$.  Hence the distortion
	constant becomes $C_0=\max\set{1/K,4K}$ (see \refeq{C0}).
	By choosing $K$ to minimize the distortion constant, this
	proves:

	\bthmT[optDistortion]{Parametric Cubic Models}
		Consider representations
			$$\mu_{n,K}: \partial [-K,K]^n\to S^n$$
		which are parametrized by $K>0$.
		\benum[(a)]
		\item ($n=3$) The optimal distortion of $C_0=2$
		is achieved when $K =1/2$.
		\item ($n=2$) The optimal distortion of
				$C_0=\sqrt{3}$
		is achieved when
				$K =1/\sqrt{3}$.
		\eenum
	\ethmT
	In practice, we would like $K$ to be a dyadic number
	(BigFloats) so that subdivision (which is typically
	reduced to bisection) can be carried
\ifarxiv{%
	out exactly without error in binary notations.
	}{%
	out exactly.
	}
	This implies we should choose $K=1/2$ when $n=2$ 
	to get a suboptimal distortion of $C_0=2$.  This remark
	is relevant for our 5DOF robots in \cite{rod-ring}. 

\ignore{MOVE:
	For $n=3$, we need to introduce an involution
	on $S^3$ to represent $SO(3)$; this involution
	results in a corresponding cubic model $\whso[3]$.
	Both $\sqs[2]$ and $\whso[3]$ appear to give us the first	
	subdivision schemes for these spaces with the
	correct global topology; they were
	implemented in our SSS framework
	\cite{rod-ring,zhang-chiang-yap:se3-robot:24}.
}
\ignore{
	It is obvious how to do subdivision in $\sqs$.
	This is illustrated in \refFig{cone-disc-ring}(b).
	After the first subdivision of $\sqs$ into 6 faces,
	subsequent subdivision is just the usual quadtree subdivision
	of each face.  We interpret the subdivision of $\sqs$
	as a corresponding subdivision of $S^2$.  In
	\cite{sss2}, we give the general framework using the notion of
	\dt{subdivision charts and atlases}
	(borrowing terms from manifold theory).
	}%

\ignore{
	In order to do subdivision in $S^2$, we introduce the
	\dt{subdivision charts and atlases}
	(borrowing terms from manifold theory).  
	First, we subdivide $\sqs$ into its 6 faces, denoted
	$S_i$ ($i\in I$) where
	$I \as \set{+,-}\times \set{x,y,z}$ serves as the
	index set for the 6 faces.  Each $i\in I$ corresponds
	to a semi-axis.  
	This is illustrated in \refFig{cone-disc-ring}(b).
	Intuitive, each $S_i$ is the root of a
	subdivision tree (actually quadtree).
	A box $B$ in one of these subdivision trees
	represents a corresponding region of $S^2$.  More precisely,
	we define six maps called \dt{charts},
	$\mu_i:[-1,1]\to S^2$ such that the range of these six charts
	covers $S^2$.   A more explicit description is to write
	$\mu_i(a,b)\as \wh{m_i(a,b)}$ where each $m_i: [-1,1]\to S_i$
	in a natural way.
	E.g., if $i=+x$ then $m_i(a,b)=(1,a,b)$,
	and $i=-y$ then $m_i(a,b)=(a,-1,b)$, etc.
	We can now do subdivision for $S^2$ via these charts
	(the initial subdivision of $S^2$ simply splits into the
	6 $S_i$'s, but subsequent splits creates four congruent
	subsquares, as in a quadtree).
	The set $\set{\mu_i: i\in I}$ charts is called
	an \dt{atlas} for $S^2$;
	see \cite{sss2} for a general treatment of 
	subdivision charts and atlases.
	}

	\ignore{
	{\bf Boxes of $\RR^3\times S^2$.}
	By ``box'' of $\RR^3\times S^2$, we mean a set
	$B\ib \cspacehat =\RR^3\times \sqs$ of the form $B=B^t\times B^r$
	where $B^t$ is the usual axis-parallel
	box in $\RR^3$, and $B^r$ is
	either equal to $\sqs$ or equal to a box of some face of $S_i$
	($i\in I$).   In the latter case, we represent $B^r$
	as a box in $[-1,1]$ together with an indicator $i\in I$.
	We call $B^t$ and $B^r$ the 
	\dt{translational} and \dt{rotational} components of $B$.
	Two $d$-dimensional boxes are \dt{adjacent} if
	their intersection is a non-degenerate $d-1$-dimensional set.
	Besides the adjacencies of boxes of $S_i$ that comes from subdivision,
	each box on the boundary of $S_i$ is also adjacent with some $S_j$
	($i\neq j$) which are relatively easy to maintain.
	%
	Let $m_{B}$ and $r_{B}$ denote the center (``midpoint'') and
	radius of $B^t$ (radius is the distance from the center to any corner
	of the cube).   In other words, $m_B$ and $r_B$
	ignores the rotational part.  There are 2 kinds of splits:
	\dt{T-split($B$)}
	splits $B^t$ into 8 congruent subboxes,
	which are paired with the unsplit $B^r$ to produce $8$ children.
	\dt{R-split($B$)}
	has 2 cases: if $B^r=\sqs$, then it splits $\sqs$ into $6$ faces.
	Otherwise, $B^r$ is a square of $\sqs$, and we split it into
	$4$ congruent squares.  In either case, each of the squares
	are then paired with $B^t$ to form the children of $B$.
	}%
	
\ignore{
\section{Cubic Model $\whso[3]$ for $SO(3)$}
	There is a natural metric on $S^3$.
	It is also known the function
		$\Phi_6(R,R')=\|\log (RR\tr)\|$ 
	for $R,R'\in SO(3)$ is a metric
	(see proof in \cite[Appendix B]{huynh:metrics-rotation:09}).
	In the following, we show more generally that any
	metric on $S^3$ induces a metric on
	$SO(3)$ via the involution map $q\mapsto -q$:
	let $(X,d)$ be a metric space with
	an involution map $-: X\to X$ (i.e., $-(-(x))=x$).
	We say $x$ and $-x$ are equivalent, denoted $x\sim -x$;
	let $\olx$ denote the equivalence class $\set{x,-x}$.
	Consider the space $\olX=\set{\olx: x\in X}$.
	We define the function $\old:\olX^2\to\RR_{\ge 0}$
	where $\old(\olx,\oly)=\min\set{d(x,y),d(x,-y)}$.

	\bleml[invol]
	If metric $d$ in invariant under the involution,
			i.e., $d(x,y)=d(-x,-y)$ for all $x,y\in X$,
	then the function $\old(\olx,\oly)$ is well-defined,
	and defines an induced metric space $(\olX,\old)$.
	\eleml
	\ignore{
	\bpf
		Suppose $x\sim x'$ and $y\sim y'$.
		Then we may verify that
			$$\min\set{d(x,y),d(x,-y)}=\min\set{d(x',y'),d(x',-y')}.$$
		This proves that $\old(\olx,\oly)$ is well-defined
		(it does not depend on the choice of representatives of
		$\olx,\oly$).
		To show that $\old$ is a metric, it is clear that
		$\old(\olx,\oly)\ge 0$ with equality iff $\olx=\oly$.
		Further, $\old(\olx,\oly)=\old(\oly,\olx)$.  To show
		the triangular inequality:
		\beqarrys
		\old(\olx,\oly)+\old(\oly,\olz)
			&=& d(\sigma x,y)+d(y,\sigma' z)
				& \text{(for some $\sigma,\sigma'\in\set{\pm 1}$)}\\
			&\ge& d(\sigma x,\sigma' z)
				& \text{(since $d$ is a metric)}\\
			&\ge& \old(\olx,\olz)
				& \text{(definition of $\old$)}.
		\eeqarrys
	\epf
	}
	Since $SO(3)$ is the equivalence classes
	of $S^3$ under the involution $q\mapsto -q$,
	and the metric on $S^3$ is invariant under this involution,
	we conclude that $SO(3)$ has a metric 
	induced by the metric on $S^3$.
	Our next question is to derive the distortion constant
	for $(\cdot)\WH$.

	\bthmT[distortion3]{Distortion of the Cubic Model of $S^3$}
		The map $(\cdot)\WH: S^3\to \sqs[3]$
		has distortion range $[1,4]$.
		It follows that the distortion constant is $C_0=4$.
	\ethmT

	For this proof, we need the tools of differential geometry,
	specifically the metric on $S^3$ is Riemannian.
	Since $\sqs[3]$ is piece-wise Riemannian, we can
	restrict our argument to each face of $\sqs[3]$ separately.
	As Riemannian spaces,
	we can reduce the argument to an infinitesimal analysis
	of distortion at each point.
	The map $(\cdot)\WH: S^3\to\sqs[3]$ induces a map $\tau_q$ 
	(for each $q\in S^3$) from the
	tangent space $T_q S^3$ at $q$
	to the tangent space $T_{\whq}\sqs[3]$ at $\whq$.
	Then the distortion range at $q$
	is the closure of
	the set $\set{\frac{\|\tau_q(X)\|}{\|X\|}: X\in T_q S^3}$.
	Taking the union over all $q\in S^3$ gives the distortion
	range of the map.  See the actual computation in
	Appendix A.

}

\section{Reduction to Distortion in Riemannian Metric}
\label{supplement}
	We now reduce the distortion of maps between metric spaces to
	the distortion of
	diffeomorphisms $F$ between Riemannian manifolds $M,N$,
		$$F : M \to N$$
	where $M,N$ are smooth $n$-dimensional manifolds.
	Our terminology and notations in differential geometry
	follow \cite{petersen:riemannian:bk,lee:manifold:bk}.
	A Riemannian manifold is a pair $(M,g_M)$ where $g_M$
	(called a \dt{Riemannian metric})
	is an inner product (positive, bilinear, symmetric)
	on vectors $u_\bfp, v_\bfp$ in the tangent space
	$T_\bfp M$ ($\bfp\in M$). 
	We write ``$g_M\bang{u_\bfp,v_\bfp}$''
	instead of $g_M(u_\bfp,v_\bfp)$ to
	suggest the inner product property. 
	Also, write
	$|u_\bfp|_{g_M} \as \sqrt{g_M\bang{u_\bfp,u_\bfp}}$.
	As $(N,g_N)$ is also a Riemannian manifold, we get an
	(induced) Riemannian metric $F^*(g_N)$ for $M$ where
		\beql{F*gN}
			F^*(g_N)\bang{u_\bfp,v_\bfp} \as
			g_N\bang{D_F(u_\bfp), D_F(v_\bfp)},\eeql
	called the \dt{$F$-pullback} of $g_N$
	\cite[p.333]{lee:manifold:bk}, and $D_F$
	is the Jacobian of $F$.
	Then we define the \dt{metric distortion range}
	of $F$, namely, the closure of the set
		\beql{MDF}
			MD_F= [m_F,M_F]
				\as \set{
					\frac{ |v_\bfp|_{F^*(g_N)}  }{  |v_\bfp|_{g_M} }
					: v_\bfp \in T_\bfp M, \bfp\in M}.\eeql
	%

	From any Riemannian manifold $(M,g_M)$, we derive
	a distance function\footnote{
		Following \cite[p.~328]{lee:manifold:bk},
		we call $d_{g_M}$ a \dt{distance function}, reserving ``metric''
		for $g_M$.}
	given by
		\beql{distance}
			d_{g_M}(\bfp,\bfq)
				\as \inf_{\bfp\lto[\pi] \bfq}\quad 
						\int_0^1 |\pi'(t)|_{g_M}dt
				 = \inf_{\bfp\lto[\pi] \bfq}\quad 
						\int_0^1 \sqrt{g_M\bang{\pi'(t),\pi'(t)}} dt
			\eeql
	where $\bfp\lto[\pi] \bfq$ means that $\pi:[0,1]\to M$
	is a smooth curve with $\pi(0)=\bfp$, $\pi(1)=\bfq$;
	also $\pi'(t)$ is the tangent vector to the curve at $\pi(t)$.  
	The pair $(M,d_{g_M})$ is now a metric space
	\cite[Theorem 13.29, p.~339]{lee:manifold:bk}.
\ifarxiv{
	In particular, the distance function $d_{F^*(g_N)}$ 
	derived from the $F$-pullback metric has this 
	characterization:


	\bleml[F*gN]
		$$ d_{F^*(g_N)}(\bfp,\bfq) = d_{g_N}(F(\bfp),F(\bfq)).$$
	\eleml

	For any $\bfp,\bfq\in M$, eonsider these two sets of paths 
		\beql{piMN}
			\mmatX{
			\Pi_M(\bfp,\bfq)\as \set{\pi
			\text{ is a path from $\bfp$ to $\bfq$ in $M$}}\\
			\Pi_N(F(\bfp),F(\bfq))\as \set{\pi
			\text{ is a path from $F(\bfp)$ to $F(\bfq)$ in $N$}}
			}
		\eeql
	Our proofs exploit the fact that there is a simple
		bijection $\pi\mapsto F\circ\pi$ that takes
	$\pi\in \Pi_M(\bfp,\bfq)$ to
	$F\circ\pi\in \Pi_N(F(\bfp),F(\bfq))$.
}{
}

\ifarxiv{
	\bpf
		\begin{align}
			d_{F^*(g_N)}(\bfp,\bfq)
				&= \displaystyle{\inf_{\bfp\lto[\pi]\bfq}} \int_0^1
						\sqrt{F^*(g_N)\bang{\pi'(t),\pi'(t)}}dt
						\nonumber\\
				&= \displaystyle{\inf_{\bfp\lto[\pi]\bfq}} \int_0^1
						\sqrt{g_N\bang{D_F(\pi'(t)),D_F(\pi'(t))}}dt
						& \text{(by \refeq{F*gN})}
						\nonumber\\
				&= \displaystyle{\inf_{\bfp\lto[\pi]\bfq}} \int_0^1
					\sqrt{g_N\bang{(F\circ\pi)'(t)),(F\circ\pi)'(t)}}dt
					\nonumber\\
				&= \displaystyle{\inf_{F(\bfp)\lto[F\circ\pi]F(\bfq)}}
						\int_0^1
					|(F\circ\pi)'(t)|_{g_N} dt
					\nonumber\\
				&= d_{g_N}(F(\bfp),F(\bfq))
						& \text{(by \refeq{distance})}
					\label{eq:FgN}
		\end{align}
		where the equality in \refeq{FgN} is based on the 
		bijection $\pi\mapsto F\circ\pi$ between the sets in
		\refeq{piMN}.
	\epf
}{%
}%

	The \dt{distance distortion range} of
		$F:(M,d_{g_M})\to (N,d_{g_N})$
	(viewed as maps between metric spaces) is
		\beql{DF}
			D_F \as \set{
				\frac{d_{g_N}(F(\bfp),F(\bfq))}{d_{g_M}(\bfp,\bfq)}
					: \bfp,\bfq\in M, \bfp\ne \bfq}.
			\eeql
	We next connect distance distortion to metric distortion of $F$:
	
	\bthmT[mdistortion]{Metric Distortion}
		\ \\ If $F:M\to N$ is a smooth map between
		two Riemannian manifolds, then the
		metric distortion range of $F$ is equal
		to the distance distortion range of $F$:
				$$MD_F = D_F$$
	\ethmT

\ifarxiv{
	\bpf
	Let $MD_F=[m_F,M_F]$ as in \refeq{MDF}.
	If $\pi:[0,1]\to M$ is a smooth path in $M$, then
		$$ m_F\cdot |\pi'(t)|_{g_M}\le 
			| (F\circ \pi)'(t)|_{g_N} \le M_F\cdot |\pi'(t)|_{g_M}.$$
	\beqarrys
		d_{g_N}(F(\bfp),F(\bfq))
			 &=& \inf \set{
			 	\int_0^1 |\pi'(t)|_{g_N} dt : \pi\in
				\pi_N(F(\bfp),F(\bfq))} \\
			 &=& \inf \set{
			 	\int_0^1 |(F\circ \pi)'(t)|_{g_N} dt : \pi\in
				\pi_M(\bfp,\bfq)}
			& \text{(with bijection $\pi\mapsto F\circ\pi$)} \\
			 &\le& \inf \set{
			 	\int_0^1 M_F |\pi'(t)|_{g_M} dt : \pi\in
				\pi_M(\bfp,\bfq)} \\
			 &=& M_F\cdot d_{g_M}(\bfp,\bfq).
	\eeqarrys
		Similarly, 
		$d_{g_N}(F(\bfp),F(\bfq)) \ge m_F \cdot d_{g_M}(\bfp,\bfq)$.
		This proves that $D_F\ib MD_F$.
		To prove equality, we note that \refThm{rangemu3}
		shows the existence of $\bfq\in B_1$,
		and $\bfv\in T_\bfq B_1$ such that
			$$ m_F\cdot |\bfv|_{g_{B_1}} = |\bfv|_{g_{S^3}}.$$
		We pick $\bfp$ arbitrarily close to $\bfq$
		and $\pi$ is a geodesic from $\bfp$ to $\bfq$ such that
				$\pi'(1)=\bfv$.
		Then
			$d_{g_{\whso[3]}}(\bfp,\bfq)
					=\int_{0}^1 |\pi'(t)|_{g_{\whso[3]}} dt$
		is arbitrarily close to
			$m_F \cdot d_{g_{S^3}}(\bfp,\bfq)$.
		A similar argument shows that the upper bound $M_F$
		is achieved.
	\epf
}{
	See proof in \cite{zhang-yap:distortion-arxiv:25}.
}
\subsection{Metric Distortion Range of $\mu_3$}
	Our \refThm{mu3} is now a consequence of 
	\refThm{mdistortion} and the following theorem.
	\bthmT[rangemu3]{Metric Distortion Range for $\mu_3$}
		\ \\
		The metric distortion range $MD_{\mu_3}$
		for $\mu_3: \whso[3]\to S^3$ is
			$$[m_{\mu_3},M_{\mu_3}]= [\tfrac{1}{4},1].$$
		Moreover, there exists
		$\bfp,\bfq\in \whso[3]$, and
		$\bfu\in T_\bfp \whso[3]$, $\bfv\in T_\bfq \whso[3]$ 
			such that
			\beql{mMmu3}
				m_{\mu_3} = 
				\frac{|\bfu|_{g_{S^3}}}{|\bfu|_{g_{\whso[3]}}},
					\qquad
				M_{\mu_3} =
				\frac{|\bfv|_{g_{S^3}}}{|\bfv|_{g_{\whso[3]}}}.
			\eeql
	\ethmT
	
	Note that our theorem gives the exact metric distortion range.
	To achieve this, we will find two expressions for the
	metric norm $|v_\bfp|_{g_N}$: one that achieves the upper
	bound, and another that achieves the lower bound.

	\bpf
	Let $\mu_3: \partial[-1,1]^4\to S^3$ where
	$\partial[-1,1]^4$ is viewed as the union of $8$ cubes,
	corresponding to each choice of $w,x,y,z=\pm1$.
	By symmetry, we focus on the cube
	$B_1=\{(w,x,y,z)\in\partial[-1,1]^4:w=1\}$.  Thus
			$$\mu_3(1,x,y,z)= \tfrac{1}{r}(1,x,y,z)$$
	where	$r=\sqrt{1+x^2+y^2+z^2}$.
	If $g_{S^3}$ is the induced Riemannian metric for $S^3$,
	then $MD_{\mu_3}$ is the range of
		$\sqrt{\frac{\mu_3^\ast(g_{S^3}) \la v_\bfp,v_\bfp\ra}
			{g_{B_1}\la v_\bfp,v_\bfp\ra}}$	
	over $v_\bfp\in T_{v_\bfp}B_1$ for all $\bfp\in B_1$
	and $\mu_3^\ast(g_{S^3})$ is pull-back metric.
	First compute the Jacobian of $\mu_3$:

		\beql{jmu3}
		D_{\mu_3}=J_{\mu_3}  =\left(\begin{array}{ccc}
			\frac{\partial(1/r)}{\partial x}
				& \frac{\partial(1/r)}{\partial y}
					& \frac{\partial(1/r)}{\partial z}\\
			\frac{\partial(x/r)}{\partial x}
				& \frac{\partial(x/r)}{\partial y}
					& \frac{\partial(x/r)}{\partial z}\\
			\frac{\partial(y/r)}{\partial x}
				& \frac{\partial(y/r)}{\partial y}
					& \frac{\partial(y/r)}{\partial z}\\
			\frac{\partial(z/r)}{\partial x} 
				& \frac{\partial(z/r)}{\partial y} 
					& \frac{\partial(z/r)}{\partial z}
			\end{array}\right) 
		 = \frac{1}{r^3}\left(\begin{array}{ccc}
				-x & -y & -z \\
				x^2-r^2 & xy & xz \\
				xy & y^2-r^2 & yz \\
				xz & yz & z^2-r^2
		\end{array}\right).
		\eeql
	In the following, we may assume that 
			$v_\bfp = (\dx,\dy,\dz)\tr \in T_\bfp B_1$
	where
		\beql{vbfp}
			g_{B_1}(v_\bfp,v_\bfp)= \la v_\bfp,v_\bfp\ra =
				\dx^2+\dy^2+\dz^2=1.
		\eeql
		\begin{align}
	\text{The pull-back metric}\quad
			\mu_3^\ast(g_{S^3}) \la v_\bfp,v_\bfp \ra\  
			&\lefteqn{= \bang{ J_{\mu_3}\bigcdot v_\bfp,
					J_{\mu_3}\bigcdot v_\bfp}
				}	 \nonumber\\
			&= v_\bfp\tr (J_{\mu_3}\tr \bigcdot J_{\mu_3})	v_\bfp
					\nonumber\\
			&= \frac{1}{r^4}
					v_\bfp\tr\bigcdot
					\mmatP{r^2-x^2 & -xy & -xz \\
						-xy & r^2-y^2 & -yz \\
						-xz & -yz & r^2-z^2 }
					\cdot v_\bfp
				& \text{(from \refeq{jmu3})} \nonumber\\
			&= \frac{E}{r^4}
		\end{align}
	where
		\begin{align*}
			E &\as (r^2-x^2)\dx^2 +(r^2-y^2)\dy^2 +(r^2-z^2)\dz^2
				- 2\big(xy\dx\dy+yz\dy\dz+xz\dx\dz\big)
		\end{align*}
	To facilitate further manipulation, rewrite $E$ in the compact form
		\beql{KEY}
			E= \sum_i (r^2-x_i^2)\dx_i^2
					- 2\sum_{i,j} x_ix_j\dx_i\dx_j
		\eeql
	where $(x,y,z)=(x_1,x_2,x_3)$ and the sums
	$\sum_i, \sum_{i,j}$ (and $\sum_{i,j,k}$) are interpreted
	appropriately:
		$i,j,k$ range over $\set{1,2,3}$, with
		$i$ chosen independently, 
		but $j$ chosen to be different from $i$,
		and $k$ chosen to be different from $i$ and $j$.
		Thus each sum has exactly 3 summands.
	
	To obtain upper and lower bounds on
		$\mu_3^\ast(g_{S^3})\la v_\bfp,v_\bfp \ra= E/r^4$, we need
		to express $E$ in two different ways:
	\benum[(A)]
	\item
		For the lower bound, 
		\begin{align*}
			E &= \sum_i (r^2-x_i^2)\dx_i^2
					- 2\sum_{i,j} x_ix_j\dx_i\dx_j\\
			 &= \sum_{i,j,k} (1+x_j^2+x_k^2)\dx_i^2
					- 2\sum_{i,j} x_ix_j\dx_i\dx_j\\
			 &= \sum_i \dx_i^2
			 	+ \sum_{i,j} (x_i^2\dx_j^2+x_j^2\dy_i^2)
					- 2\sum_{i,j} x_ix_j\dx_i\dx_j\\
			 &= 1 + \sum_{i,j}(x_i\dx_j-x_j\dx_i)^2
						& \text{(as $1=\sum_i\dx_i^2$)}\\
		\end{align*}
		Hence,
		\ignore{
		\begin{align*}
			\mu_3^\ast\la v_\bfp,v_\bfp \ra
			&= \frac{1+ \sum_{i,j} (x_i\dy_j-y_j\dx_i)^2}
					{r^4} \\
			&\ge	\frac{1}{(1+x^2+y^2+z^2)^2}  \\
			&\ge	\frac{1}{16}.
		\end{align*}
		}
		$$ \mu_3^\ast(g_{S^3})\la v_\bfp,v_\bfp \ra
			= \frac{1+ \sum_{i,j} (x_i\dy_j-y_j\dx_i)^2}
					{r^4} 
			\ge	\frac{1}{(1+x^2+y^2+z^2)^2}  
			\ge	\frac{1}{16}.$$
		The last two inequalities become equalities when
		$x=y=z=1$ and $\dx=\dy=\dz=1/\sqrt{3}$.
		This proves a tight lower bound of $1/16$ for
			$\mu_3^\ast(g_{S_3})\la v_\bfp,v_\bfp \ra$.
	\item
		To obtain an upper bound, we rewrite $E$ as follows:
		\begin{align*}
		E &= \sum_{i,j,k} (1+x_j+x_k)\dx_i^2
				- 2\sum_{i,j} x_ix_j\dx_i\dx_j \\
		&= \sum_i\dx_i^2 + \sum_{i,j,k} x_i^2(\dx_j^2+\dx_k^2)
				- 2\sum_{i,j} x_ix_j\dx_i\dx_j \\
		&= 1+ \sum_{i} x_i^2(1-\dx_i^2)
				- 2\sum_{i,j} x_ix_j\dx_i\dx_j 
						& \text{(as $1=\sum_i\dx_i^2$)}\\
		&= \Big(1+\sum_i x_i^2\Big)-\sum_i x_i^2\dx_i^2
				- 2\sum_{i,j} x_ix_j\dx_i\dx_j \\
		&= r^2 -\Big(\sum_i x_i\dx_i\Big)^2 
		\end{align*}
	\ignore{
		\begin{align*}
			\mu_3^\ast(g_{S_3})\la v_\bfp,v_\bfp \ra
			&=	\frac{ r^2 -(x\dx+y\dy+z\dz)^2}
					{r^4} \\
			&\le	\frac{1}{r^2}  \\
			&\le	1.
		\end{align*}
	}%
		$$\mu_3^\ast(g_{S_3})\la v_\bfp,v_\bfp \ra
			=	\frac{ r^2 -(x\dx+y\dy+z\dz)^2}
					{r^4} \\
			\le	\frac{1}{r^2}  \\
			\le	1.$$
	The last two inequalities
	are equalities when $x=y=z=0$.
	\eenum
	We have therefore established that 
		$$\sqrt{\mu_3^\ast(g_{S_3})\la v_\bfp,v_\bfp \ra}
			\in [\tfrac{1}{4},1]$$
		and these bounds are achievable.
	\epf
	
	With the above compact notation, we could generalize
\ifarxiv{%
	the argument to $\mu_n$ (see \cite{zhang:thesis}), showing
	}{%
	the argument to $\mu_n$
	(see \cite{zhang-yap:distortion-arxiv:25}), showing
	}
	that $MD_{\mu_n} =[1/(n+1),1]$.


\ignore{ 

\blbar
TEMPORARY NO IGNORE:
\bpf
\cored{THIS SHOULD BE GREATLY STREAMLINED!!!}

	Given a \Riman\ manifold $(M,g_M)$,
	the \dt{length} of a tangent vector $v_p\in T_pM$ is
		\[
		|v_p|=\sqrt{g_M\la v_p,v_p\ra}.\]
	The length of a curve $\pi:[0,1]\to M$ is defined as
		\[
		\len(\pi) =\int_0^1|\pi'(t)|\ddt
			=\int_0^1\sqrt{g_M\la\pi'(t),\pi'(t)\ra}\ddt.\]
	Given $p,q\in M$, the distance between $p$ and $q$ is defined as the
	length of the shortest curve connecting $p$ and $q$, i.e.,
		\begin{equation}\label{curve_length}
		d_M(p,q) =\inf_{\pi:[0,1]\to M,\pi(0)
				=p,\pi(1)=q}\len(\pi).
		\end{equation}
	The metric on $M$ is induced by this distance function $d_M$. When
	the manifolds are $\Ws$ and $\Xs$, the metrics $\disW$ and $\disX$
	are also defined by (\refEq{curve_length}). Therefore, the distortion
	between $\bfp,\bfq\in\disW$ where $\mu(\bfp)=\gamma$,
	$\mu(\bfq)=\zeta$ is
		\[\frac{\disX(\gamma,\zeta)}{\disW(\bfp,\bfq)}
			= \frac{\inf\int_0^1\sqrt{g_\Xs\la
			(\mu\circ\pi)'(t),
			(\mu\circ\pi)'(t)\ra}\ddt}{\inf\int_0^1\sqrt{g_\Ws\la\pi'(t),
			\pi'(t)\ra}\ddt}
		= \frac{\inf\int_0^1\sqrt{\mu^\ast(g_\Xs)\la
			\pi'(t),
			\pi'(t)\ra}\ddt}{\inf\int_0^1\sqrt{g_\Ws\la\pi'(t),
			\pi'(t)\ra}\ddt}.\]

	We define $\ltor_\mu$ and $\utor_\mu$ as follows:

	\[\ltor_\mu \as \inf_{\bfb\in\Ws,v_\bfb\in T_\bfb\Ws}
		\sqrt{ \frac{\mu^\ast(g_\Xs)\la v_\bfb,v_\bfb\ra}
			{g_\Ws\la v_\bfb,v_\bfb\ra} }
	\]
	and
	\[\utor_\mu \as \sup_{\bfb\in\Ws,v_\bfb\in T_\bfb\Ws}
		\sqrt{ \frac{\mu^\ast(g_\Xs)\la v_\bfb,v_\bfb\ra}
			{g_\Ws\la v_\bfb,v_\bfb\ra} }.
	\]
	
	\begin{theorem}\label{tor_bound}
		\[D_\mu = [\ltor_\mu,\utor_\mu].\]
		i.e. $\ltor(\mu)=\ltor_\mu$ and $\utor(\mu)=\utor_\mu$.
	\end{theorem}
	
	\begin{proof}
	For any two $\bfp,\bfq\in \Ws$ and
	$\mu(\bfp)=\gamma,\mu(\bfq)=\zeta\in\Xs$, we consider any path
	$\pi:[0,1]\to\Ws$ such that $\pi(0)=\bfp$ and $\pi(1)=\bfq$.
	
	As $(\mu\circ\pi)'(t)$ and $\pi'(t)$ are tangent vectors on $\Xs$ and
	$\Ws$
	respectively, by definition, they satiesfies
	\[\ltor_\mu\leq\frac{|(\mu\circ\pi)'(t)|}{|\pi'(t)|}\leq\utor_\mu\]
	As an	estimation of the curve length:
	\[\len(\mu\circ\pi) =
					\int_0^1|(\mu\circ\pi)'(t) |\ddt
			\leq \utor_\mu\int_0^1|\pi'(t)|\ddt=\utor_\mu\len(\pi)\]
	and
		\[\len(\mu\circ\pi) =\int_0^1| (\mu\circ\pi)'(t)|\ddt
				\geq \ltor_\mu\int_0^1|\pi'(t)|\ddt=\ltor_\mu\len(\pi)\]
	
	The inequalities are true for all such curves, and therefore, after
	we take $\inf$ of $\len(\mu\circ\pi)$ over all such $\pi$, we still
	have
	\[\disX(\gamma,\zeta) = \inf_\pi\len(\mu\circ\pi)
			\leq\utor_\mu\inf_\pi\len(\pi)
			=\utor_\mu\disW(\bfp,\bfq)\]
	and
	\[\disX(\gamma,\zeta) = \inf_\pi\len(\mu\circ\pi)
		\geq\ltor_\mu\inf_\pi\len(\pi)
		=\ltor_\mu\disW(\bfp,\bfq).\]
	The equalities reach when $\bfp$ and $\bfq$ are near enough.
	
	Therefore, $\ltor(\mu)=\ltor_\mu$ and $\utor(\mu)=\utor_\mu$,
	thus proving
	\refThem{tor_bound}.
	\end{proof}
\epf
\elbar
}%
\ignore{
	We want to compute the distortion range of the map
		\beql{mu3}
			\mu_3: \whS^3 \to S^3.	\eeql
	The Riemannian metric $g_{S^3}$
	on $S^3$ is the canonical one
	induced from the embedding $S^3\ib \RR^4$
	\cite[Example 1.1.3, p.~3]{petersen:riemannian:bk}.
	It follows that there is an metric
		\beql{pullback-metric}
			\mu_3^*(g_{S^3}) \eeql
	on $\whS^3$ which is induced by the pullback of $\mu_3$
	\cite[p.~3]{petersen:riemannian:bk}.
	On the other hand, viewing $\whS^3$
	as a piecewise-linear (PL) polyhedral manifold
	\cite{rourke-sanderson:PL-bk}
	endowed with the Riemannian metric
		$g_{\whS^3}$
	induced by the embedding $\whS^3\ib \RR^4$
	(cf.~\cite{miller-pak:unfolding:08}).
	Call the $g_{\whS^3}$
	the natural \dt{Euclidean} Riemannian metric
	of $\whS^3$.
	Thus, we have two Riemannian
	metrics on $\whS^3$: the natural one and the induced one:
		\beql{riemannianMetrics}
			g_{\whS^3},\qquad \mu_3^*(g_{S^3})
			\eeql
	They give rise to two metrics on $\whS^3$, denoted
		$$d_{\RR^3}(\bfp,\bfq),\qquad
				d_{S^3}(\bfp,\bfq)$$
	respectively.  By THEOREM XXX,
	the distortion range of $\mu_3$
	is given by $[m(\mu_3), M(\mu_3)]$.


	Representation map is a locally homeomorphism $\mu:\Ws\to\Xs$ such
	that its inverse is an embedding restricted to each chart in an
	atlas. Given $\bfp,\bfq\in\Ws$ where $\gamma=\mu(\bfp)$ and
	$\zeta=\mu(\bfq)$, the \dt{distortion} of $\mu$ between $\bfp$ and
	$\bfq$ is
	\[D_\mu(\bfp,\bfq)=\frac{\disX(\gamma,\zeta)}{\disW(\bfp,\bfq)}.\]
	The range of this ratio for $\bfp,\bfq\in\Ws$ is the \dt{distortion}
	of $\mu$, denoted by $D_\mu$ (\cored{This is the distortion range
	$D_\mu$ defined in chapter 2 in this paper ima2025}). i.e.,
	\[ D_\mu=\left\{\tau>0 :
	\inf_{\bfp,\bfq\in\Ws}\frac{\disX(\gamma,\zeta)}
	{\disW(\bfp,\bfq)}\leq\tau\leq\sup_{\bfp,\bfq\in\Ws}
		\frac{\disX(\gamma,\zeta)}{\disW(\bfp,\bfq)}\right\}.\]
	For simplicity, we define \dt{lower bound}
	$\ltor(\mu) =
		\inf_{\bfp,\bfq\in\Ws}
		\frac{\disX(\gamma,\zeta)}   
		{\disW(\bfp,\bfq)}$
	and \dt{upper bound}
		$\utor(\mu) =\sup_{\bfp,\bfq\in\Ws}
			\frac{\disX(\gamma,\zeta) 
			}{\disW(\bfp,\bfq)}$.
	Then, $D_\mu=[\ltor(\mu),\utor(\mu)]$.
	
	A \dt{distortion bound} is a number $C_0\geq1$ such that
	$\tor(\mu)\ib[\frac{1}{C_0},C_0]$. If the distortion $\tor(\mu)$ has
	a bound $C_0$, then for all $\bfp,\bfq\in\Ws$, we have
	\[\frac{1}{C_0}\disW(\bfp,\bfq)
		<\disX(\gamma,\zeta)<C_0 \disW(\bfp,\bfq),\]
	for $\gamma=\mu(\bfp),\zeta=\mu(\bfq)\in\Xs$.
	
	We use $\len(\pi)$ to denote the length of a path $\pi$.
}	
\ignore{
\subsection{Distortion for $SO(3)$}
	Finally, consider the distortion for the
	representation $\olmu_3: \sqs[3]\to SO(3)$,
	where $\olmu_3=\mu_3\circ \Phi_3$ where
	$\rho: S^3\to SO(3)$ is usual map from unit
	quaternions to $SO(3)$.  Huynh observed that
		$\Phi_6(\rho(q_1),\rho(q_2))=2\Phi_3(q_1,q_2)$
	\cite{huynh:metrics:rotation:90}.
	Since $\Phi_3, \Phi_6$ are our ``natural'' distance functions
	on $S^3, SO(3)$, the distortion range of $\rho$ is $[2,2]$.
	By \refLem{composition},
	we conclude that $D_{\olmu_3}=2[1/4,1] = [\half,2]$.
}
\ignore{
	Finally, we consider the distortion for the
	representation $\olmu_3: \sqs[3]\to SO(3)$.
	which is derived from $\mu_3$ by the identifications
	$q=-q$ for $q\in\whso[3]$ as well as $q\in SO(3)$.

	In the following, we consider a general situation where
	the metric space $(X,d_X)$ has a
	involution map $-: X\to X$ (i.e., $-(-(x))=x$).
	We say $x$ and $-x$ are equivalent, denoted $x\sim -x$.
	Let $\olx=\set{x,-x}$ be the equivalence class of $x$.
	The quotient space $\olX=\set{\olx: x\in X}$ has the
	induced the function $\old_X:\olX^2\to\RR_{\ge 0}$
	where $\old_X(\olx,\oly)\as \min\set{d(x,y),d(x,-y)}$.
	We have the following easy result:

	\bleml[invol]
	If metric $d$ in invariant under the involution,
			i.e., $d(x,y)=d(-x,-y)$ for all $x,y\in X$,
	then the function $\old_X(\olx,\oly)$ is well-defined,
	and induces a metric space $(\olX,\old)$.
	\eleml
	\ignore{
	\bpf
		Suppose $x\sim x'$ and $y\sim y'$.
		Then we may verify that
			$$\min\set{d(x,y),d(x,-y)}=\min\set{d(x',y'),d(x',-y')}.$$
		This proves that $\old(\olx,\oly)$ is well-defined
		(it does not depend on the choice of representatives of
		$\olx,\oly$).
		To show that $\old$ is a metric, it is clear that
		$\old(\olx,\oly)\ge 0$ with equality iff $\olx=\oly$.
		Further, $\old(\olx,\oly)=\old(\oly,\olx)$.  To show
		the triangular inequality:
		\beqarrys
		\old(\olx,\oly)+\old(\oly,\olz)
			&=& d(\sigma x,y)+d(y,\sigma' z)
				& \text{(for some $\sigma,\sigma'\in\set{\pm 1}$)}\\
			&\ge& d(\sigma x,\sigma' z)
				& \text{(since $d$ is a metric)}\\
			&\ge& \old(\olx,\olz)
				& \text{(definition of $\old$)}.
		\eeqarrys
	\epf
	}
	Since $SO(3)$ is the equivalence classes
	of $S^3$ under the involution $q\mapsto -q$,
	and the metric on $S^3$ is invariant under this involution,
	we conclude that $SO(3)$ has a metric 
	induced by the metric on $S^3$.
}%
\ignore{
\subsection{Atlas Constant for an $SE(3)$ \cs}
	The isometry group for Euclidean space is
	$(\RR^n, g_{\RR^n}) = \RR^n \rtimes O(n)$ 
	where $H\rtimes G$ denote the semi-direct product
	of two groups.
	See \cite[Example 1.3.1, p.~9]{petersen:riemannian:bk}.

	Suppose that $\mu: \RR^3\times\partial[-1,1]^4\to \RR^3\times SO(3)$
	is the 
	representation map for the Delta robot. Let
	$\mu^r:\partial[-1,1]^4\to SO(3)$ be the 
	rotation representation.
	
	There is a double covering map from $S^3$ to $SO(3)$. The box
	subspace $\Ws^r=\wh{SO}(3)=\partial[-1,1]^4$ is also a double
	covering map from
	$\partial[-1,1]^4$ but whose metric is inherited from the double
	covering. So the metric on $\wh{SO}(3)$ is the same as the metric on
	$\partial[-1,1]^4$. Their relation can be given by the following
	commutative diagram, we will follow the symbols along the diagram
	throughout this section:

	\begin{center}
	\includegraphics[width=0.4\textwidth]{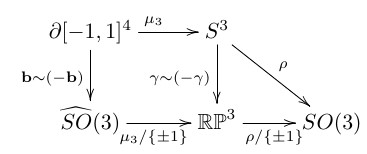}
	\end{center}
	  
	The representation $\mu^r$ is the composition of $\rho$ and $\mu_3$, 
	and we have computed $D_{\mu_3}=[\frac{1}{4},1]$. However, 
	the distortion bound for $\rho$ is indeterminate without an explicit
	metric 
	on $SO(3)$. The choice of an essential metric on $SO(3)$ is a problem.
	
	In Huynh's paper
	\cite{huynh:metrics-rotation:09},
	defines $6$ different metrics $(\Phi_i, i=1\dd 6)$ on $SO(3)$.
	The metric $\Phi_1$ is the Euclidean distance between Euler angles 
	$\|(\Delta\alpha,\Delta\beta,\Delta\gamma)\|_2$, 
	$\Phi_2$ is the norm of the difference of quaternions ($S^3$) 
	$\min\{|q_1-q_2|,|q_1+q_2|\}$, $\Phi_3$ is the angle between unit
	quatenions 
	$\arccos|q_1\cdot q_2|$, $\Phi_4$ is one minus the absolute value of
	inner product 
	of unit quaterions $1-|q_1\cdot q_2|$, $\Phi_5$ is the \Frob\ norm of
	identity matrix 
	minus the transition matrix $\|I-R_1R_2^T\|_F$, and $\Phi_6$ is the
	distance 
	by the exponential map $\|\log(R_1R_2^T)\|$. Note that $\Phi_3$ is
	the metric 
	induced by $\rho$, i.e., $\Phi_3(q_1,q_2)=g_{S^3}(q_1,q_2)$, and 
	we choose $\Phi_6$ as the \dt{essential metric} on $SO(3)$. 
	This $\Phi_6$ will be used to compute both atlas constant $C_0$ and 
	\Lip\ constant $L_0$.
	Huynh noted that $\Phi_6=2\Phi_3$, therefore $D_\rho=\{2\}$ and 
	the distortion for $\mu^r=\rho\circ\mu_3$ is $D_{\mu^r}=[\half,2]$. 

	The representation $\mu^t$ of the translational subspace is the
	identity map from 
	$\RR^3$ to itself. The distortion is obviously $1$. There are
	different ways to define 
	metrices on the combined space $\wh{SE}(3)$ and $SE(3)$. Usually,
	they are defined 
	by some $L^p$ combination for some $p\geq1$:
		\begin{align*}
		d_{\wh{SE}(3)}((\bfx,\gamma),(\bfy,\zeta))
		& = \left(d_{\RR^3}(\bfx,\bfy)^p 
			+ \lambda d_{\wh{SO}(3)}(\gamma,\zeta)\right)^{\frac{1}{p}} \\
				d_{SE(3)}((\bfx,\gamma),(\bfy,\zeta))
		& = \left(d_{\RR^3}(\bfx,\bfy)^p 
			+ \lambda d_{SO(3)}(\gamma,\zeta)\right)^{\frac{1}{p}}
		\end{align*}
	When computing the combined distortions, we could apply
	\begin{align*}
		D_{\mu} & \leq \sup \frac{d_{SE(3)}((\bfx,\gamma),(\bfy,\zeta))}
		{d_{\wh{SE}(3)}((\bfx,\gamma),(\bfy,\zeta))} \\
		& \leq \sup \left(\max\{\sup D_{\mu^t}^p,\sup D_{\mu^r}^p\}
				\right)^{\frac{1}{p}} \\
		& \leq \max\{\sup D_{\mu^t},\sup D_{\mu^r}\} \\
		& \leq 2
	\end{align*}
	and
	\begin{align*}
		D_{\mu} & \geq \inf \frac{d_{SE(3)}((\bfx,\gamma),(\bfy,\zeta))}
		{d_{\wh{SE}(3)}((\bfx,\gamma),(\bfy,\zeta))} \\
		& \geq \inf \left(\min\{\inf D_{\mu^t}^p,\inf
		D_{\mu^r}^p\}\right)^{\frac{1}{p}} \\
		& \geq \min\{\inf D_{\mu^t},\inf D_{\mu^r}\} \\
		& \geq \frac{1}{2}
	\end{align*}
	Therefore, the combined distortion is still $D_{\mu}=[\frac{1}{2},2]$.
}
%
\sect{Final Remarks}
	1. It remains to determine the distortion of the
	representation $\olmu_3: \sqs[3]\to SO(3)$,
	where $\olmu_3=\mu_3\circ \rho$ and
	$\rho: S^3\to SO(3)$ is the usual map from unit
\ifarxiv{%
	quaternions to $SO(3)$ \cite{bottema-roth:bk}.
	}{%
	quaternions to $SO(3)$.
	}
	Huynh observed that
		$\Phi_6(\rho(q_1),\rho(q_2))=2\Phi_3(q_1,q_2)$
	\cite{huynh:metrics-rotation:09}.
	Since $\Phi_3$ and $ \Phi_6$ are natural distance functions
	on $S^3$ and $SO(3)$, the distortion range of $\rho$ is $[2,2]$.
	By \refLem{composition},
	we conclude that the distance distortion
	of $\olmu_3$ is
		$$D_{\olmu_3}=2[\tfrac{1}{4},1] = [\half,2].$$

	2. This work establishes the exact distortion bounds
	on the cubic model representation $\olmu_3$ of $SO(3)$.
	A critical step was to exploit
	Riemannian geometry by reducing 
	distance distortion to metric distortion.
	We expect our cubic representation to have
	other applications, e.g., a rigorous subdivision search for an
	$\veps$-optimal rotation to best fit experimental data as in motion
	capture research.


\ignore{
	We have initiated a systematic study of
	the cubic model $\sqs[3]$ for $SO(3)$, with the
	goal of developing it into a practical 
	tool in subdivision algorithms for $SO(3)$
	or $SE(3)$.  We obtained the optimal 
	distortion constant of the cubic model; one can investigate
	other polyhedral models of $SO(3)$ or $S^2$.
	E.g., for for $\sqs[2]$, we can consider the
	truncated cube, 
	with 6 octagonal and 8 triangular faces 
	obtained by cutting off the vertices of maximum
	distortion in $\sqs[2]$.  The distortion constant is improved
	but subdivision of this model would be more challenging.

	Our framework of subdivision atlases
	for manifolds suggests many open questions for
	further exploration. 
	We also plan to implement the data-structure
	described in this paper.

}


%

\begin{thebibliography}{10}

\bibitem{bottema-roth:bk}
O.~Bottema and B.~Roth.
\newblock {\em Theoretical Kinematics}.
\newblock Dover Publications, 1990.

\bibitem{brooks-perez:subdivision:83}
R.~A. Brooks and T.~Lozano-Perez.
\newblock A subdivision algorithm in configuration space for findpath with
  rotation.
\newblock In {\em Proc. 8th IJCAI, Vol.2},
pp.~799--806, San Francisco, 1983.

\bibitem{canny:thesis}
J.~F. Canny.
\newblock {\em The complexity of robot motion planning}.
\newblock ACM Doctoral Dissertion Award Series. The MIT Press,
Cambridge, MA, 1988.
\newblock PhD thesis, M.I.T.

\bibitem{choset-etal:bk}
H.~Choset, K.~M. Lynch, S.~Hutchinson, G.~Kantor,
	W.~Burgard, L.~E. Kavraki, and S.~Thrun.
\newblock {\em Principles of Robot Motion: Theory, Algorithms, and
  Implementations}.
\newblock MIT Press, Boston, 2005.

\bibitem{dayan-solovey-pavone-halperin:near-opt:21}
D.~Dayan, K.~Solovey, M.~Pavone, and D.~Halperin.
\newblock Near-optimal multi-robot motion planning with finite sampling.
\newblock {\em IEEE ICRA},
  39(5):9190--9196, 2021.

\bibitem{hss:motionplanning:crc:17}
D.~Halperin, O.~Salzman, and M.~Sharir.
\newblock Algorithmic motion planning.
\newblock In {\em Handbook of
Disc. and Comp. Geom.}, chapter~50. Chapman \& Hall/CRC, Boca
Raton, FL, 3rd edition, 2017.

\bibitem{rod-ring}
C.-H. Hsu, Y.-J. Chiang, and C.~Yap.
\newblock Rods and rings: Soft subdivision planner for
{{\bf R}\^{}3 x {\bf S}\^{}2}.
\newblock In {\em Proc. 35th Symp. on Comp. Geometry},
pp.~43:1--43:17, 2019.


\bibitem{hsu-latombe-kurniawati:foundations:06}
D.~Hsu, J.-C. Latombe, and H.~Kurniawati.
\newblock On the probabilistic foundations of probabilistic roadmap
planning.
\newblock {\em lJRR}, 25(7):627--643, 2006.

\bibitem{huynh:metrics-rotation:09}
D.~Q. Huynh.
\newblock Metrics for {3D} rotations: Comparison and analysis.
\newblock {\em J. Math. Imaging Vis.}, 35:155--164, 2009.

\bibitem{prm}
L.~Kavraki, P.~{\v{S}vestka}, C.~Latombe, and M.~Overmars.
\newblock Probabilistic roadmaps for path planning in high-dimensional
  configuration spaces.
\newblock {\em IEEE Trans. Robotics and Automation}, 12(4):566--580, 1996.

\bibitem{lavalle:planning:bk}
S.~M. LaValle.
\newblock {\em Planning Algorithms}.
\newblock Cambridge University Press, Cambridge, 2006.

\bibitem{rrt}
S.~M. LaValle and J.~J. Kuffner~Jr.
\newblock Randomized kinodynamic planning.
\newblock {\em IJRR},
20(5):378--400, 2002.

\bibitem{lee:manifold:bk}
J.~M. Lee.
\newblock {\em Introduction to Smooth Manifolds}.
\newblock Springer, 2nd ed., 2012.

\bibitem{luo-chiang-lien-yap:link:14}
Z.~Luo, Y.-J. Chiang, J.-M. Lien, and C.~Yap.
\newblock Resolution exact algorithms for link robots.
\newblock In {\em WAFR '14}, vol.~107, 
  pp.~353--370, 2015.



\bibitem{mobileRobot-survey:21}
S.R. S\'{a}anchez-Iba\~{n}ez,
C.~P\'{e}rez-Del-Pulgar, and A.~Garc\'{i}a-Cerezo.
\newblock Path planning for autonomous mobile robots: A review.
\newblock {\em Sensors (Basel)}, 21(23)(7898), 2021.

\bibitem{nowakiewicz:mst:10}
M.~Nowakiewicz.
\newblock {MST-Based} method for {6DOF} rigid body motion planning in
	narrow passages.
\newblock In {\em Proc. Intelligent Robot Systems},
pp.~5380--5385, 2010.

\bibitem{odun-sharir-yap:retraction:83}
C.~{\'O}'D{\'u}nlaing, M.~Sharir, and C.~K. Yap.
\newblock Retraction: a new approach to motion-planning.
\newblock {\em STOC}, 15:207--220, 1983.



\bibitem{odun-yap:disc:85}
C.~{\'O}'D{\'u}nlaing and C.~K. Yap.
\newblock A ``retraction'' method for planning the motion of a disc.
\newblock {\em J.~Algorithms}, 6:104--111, 1985.

\bibitem{petersen:riemannian:bk}
P.~Petersen.
\newblock {\em Riemannian Geometry}.
\newblock Graduate Text in Math. Springer, 3rd ed., 2016.

\bibitem{ss1}
J.~T. Schwartz and M.~Sharir.
\newblock On the piano movers' problem: {I.} the case of a two-dimensional
  rigid polygonal body moving amidst polygonal barriers.
\newblock {\em Communications on Pure and Applied Mathematics},
36:345--398, 1983.

\bibitem{ss2}
J.~T. Schwartz and M.~Sharir.
\newblock On the piano movers' problem: {II.} {G}eneral techniques for
  computing topological properties of real algebraic manifolds.
\newblock {\em Advances in Appl. Math.}, 4:298--351, 1983.

\bibitem{selig:robotics:bk}
J.~Selig.
\newblock {\em Geometric Fundamentals of Robotics}.
\newblock Springer, second ed., 2005.

\bibitem{sss1}
C.~Wang, Y.-J. Chiang, and C.~Yap.
\newblock On soft predicates in subdivision motion planning.
\newblock {\em Comput. Geometry: Theory and Appl.
	(Special Issue of SoCG'13)}, 48(8):589--605, Sept. 2015.

\bibitem{sss}
C.~Yap.
\newblock Soft subdivision search in motion planning.
\newblock In {\em Proc. 1st Workshop on Robotics
  Challenge and Vision (RCV 2013)}, 2013.
\newblock A Computing Community Consortium (CCC)
{\bf Best Paper Award},
  RSS Conference 2013, Berlin. In arXiv:1402.3213.

\bibitem{sss2}
C.~Yap.
\newblock Soft subdivision search and motion planning, {II}: {A}xiomatics.
\newblock In {\em Frontiers in Algorithmics}, vol.~9130
	of {\em LNCS}, pp.~ 7--22. Springer, 2015.
\newblock Plenary talk, 9th FAW. Guilin, China.


\bibitem{yap-luo-hsu:thicklink:16}
C.~Yap, Z.~Luo, and C.-H. Hsu.
\newblock Resolution-exact planner for thick non-crossing 2-link robots.
\newblock In {\em Proc.~12th WAFR 2016},
	pp.~576--591. Springer, 2020.

\bibitem{yap:amp:87}
C.~K. Yap.
\newblock Algorithmic motion planning.
\newblock In J.~Schwartz and C.~Yap, editors,
{\em Advances in Robotics, Vol.
  1: Algorithmic and Geometric Issues}, vol.~1, pp.~95--143. Lawrence
  Erlbaum Associates, 1987.

\bibitem{zhang:thesis}
Z.~Zhang.
\newblock {\em Theory and Explicit Design of an SE(3) Robot}.
\newblock {Ph.D.} thesis, New York University, Dec. 2024.

\bibitem{zhang-chiang-yap:se3-arxiv:24}
Z.~Zhang, Y.-J. Chiang, and C.~Yap.
\newblock Theory and explicit design of a path planner for an {SE}(3)
	robot, arXiv:2407.05135, December 2024.
\newblock Includes five appendices A--E.

\bibitem{zhang-chiang-yap:se3-robot:24}
Z.~Zhang, Y.-J. Chiang, and C.~Yap.
\newblock Theory and explicit design of a path planner for an {SE(3)}
robot.
\newblock In {\em Proc.~16th WAFR},
  Springer Tracts in Advanced Robotics (STAR). Springer, 2024.


\bibitem{zhou-chiang-yap:complex-robot:20}
B.~Zhou, Y.-J. Chiang, and C.~Yap.
\newblock Soft subdivision motion planning for complex planar robots.
\newblock {\em Computational Geometry}, 92, Jan. 2021.
\newblock Article 101683.

\end{thebibliography}
%

%
\end{document}